\documentclass[letterpaper, 10 pt, conference]{ieeeconf}  % Comment this line out if you need a4paper

\IEEEoverridecommandlockouts                              % This command is only needed if 
\usepackage{amssymb}
\usepackage{amsmath}
\usepackage[english]{babel}
\usepackage{float}
\usepackage{graphicx}
\usepackage{hyperref}
\usepackage{mathtools}

\usepackage{filecontents}
\usepackage[noadjust]{cite}
\usepackage[font=footnotesize,labelfont=footnotesize]{subcaption}
\usepackage[font=footnotesize,labelfont=footnotesize]{caption}
\usepackage{subcaption}
\usepackage{caption}
\usepackage{comment}
\usepackage{authblk}
\usepackage{multirow}
\usepackage{environ}
\usepackage{graphicx}

\usepackage{amsthm, algorithm, algorithmicx} %amsthm, 
\usepackage[noend]{algpseudocode}
\usepackage[english]{babel}

\let\labelindent\relax

\usepackage[inline]{enumitem}

\graphicspath{ {img/} }

\theoremstyle{plain}  % to use together with amsthm package
\newtheorem{theorem}{Theorem}
\newtheorem{lemma}[theorem]{Lemma}
\newtheorem{definition}[theorem]{Definition}

\algnewcommand\algorithmicinput{\textbf{Input:}}
\algnewcommand\algorithmicoutput{\textbf{Output:}}
\algnewcommand\Input{\item[\algorithmicinput]}%
\algnewcommand\Output{\item[\algorithmicoutput]}%

\DeclareMathOperator*{\argmin}{arg\,min}
\DeclareMathOperator*{\argmax}{arg\,max}

\newcommand{\norm}[1]{\left\lVert#1\right\rVert}

\usepackage{xcolor}

\newcommand{\remove}[1]{}

\title{\LARGE \bf
Behavior Mixing with Minimum Global and Subgroup Connectivity Maintenance for Large-Scale Multi-Robot Systems
}

\author{Wenhao Luo, Sha Yi, and Katia Sycara% <-this % stops a space
\thanks{This work was funded by the DARPA Cooperative Agreement No.: HR00111920012, AFOSR awards FA9550-18-1-0097, FA9550-18-1-0251 and FA9550-15-1-0442.}% <-this % stops a space
\thanks{The authors are with the Robotics Institute, Carnegie Mellon University, Pittsburgh, PA 15213, USA. Email: {\tt\small \{wenhao, shayi, katia\}@cs.cmu.edu}.}%
}

\begin{document}

\maketitle
\thispagestyle{empty}
\pagestyle{empty}

%%%%%%%%%%%%%%%%%%%%%%%%%%%%%%%%%%%%%%%%%%%%%%%%%%%%%%%%%%%%%%%%%%%%%%%%%%%%%%%%
\begin{abstract}

In many cases the multi-robot systems are desired to execute simultaneously multiple behaviors with different controllers, and sequences of behaviors in real time, which we call \textit{behavior mixing}. Behavior mixing is accomplished when different subgroups of the overall robot team change their controllers to collectively achieve given tasks while maintaining connectivity within and across subgroups in one connected communication graph. In this paper, we present a provably minimum connectivity maintenance framework to ensure the subgroups and overall robot team stay connected at all times while providing the highest freedom for behavior mixing. In particular, we propose a real-time distributed Minimum Connectivity Constraint Spanning Tree (MCCST) algorithm to select the minimum inter-robot connectivity constraints preserving subgroup and global connectivity that are \textit{least likely to be violated} by the original controllers. With the employed safety and connectivity barrier certificates for the activated connectivity constraints and collision avoidance, the behavior mixing controllers are thus minimally modified from the original controllers. We demonstrate the effectiveness and scalability of our approach via simulations of up to 100 robots with multiple behaviors.

\end{abstract}

%%%%%%%%%%%%%%%%%%%%%%%%%%%%%%%%%%%%%%%%%%%%%%%%%%%%%%%%%%%%%%%%%%%%%%%%%%%%%%%%
\section{Introduction}

The ability of collaboration in multi-robot systems often relies on the local information sharing and interaction among networked robot members through a connected communication graph, e.g. flocking \cite{olfati2006flocking, zavlanos2007flocking}, formation control \cite{li2018formally} and leader selection \cite{luo2015asynchronous, luo2016distributed}. As robots are often assumed to interact in a proximity-limited manner due to limited communication range, it is necessary to consider connectivity maintenance that ensures robots stay connected as one component. This is often referred as maintaining \textit{global connectivity} \cite{sabattini2013decentralized, giordano2011bilateral,yang2010decentralized, williams2015global, tateo2018multiagent,ijcai2019-21, khateri2019comparison, luo2019voronoi, luo2019minimum} and achieved by constraining inter-robot distance while executing original tasks.
In many situations, however,
it may be more appropriate and efficient to have the 
multi-robot systems \textit{simultaneously} performing multiple behaviors in different subgroups while remaining connected. 
For example, having a robot team split into multiple operating subgroups 
%to form multiple different formations or 
flocking to multiple task areas at the same time. 
%to track different targets or monitor different task areas. 
As subgroups may be formed based on the particular combinations of robots with heterogeneous capabilities, when the robot team spreads out over multiple widely separated task areas, robots in the same subgroup for a designated task area are expected to stay \emph{locally connected by themselves} as one coherent component for efficient local collaboration. 
Global connectivity is still required to allow for global coordination among different subgroups, e.g. redistribution of robots due to dynamic task reallocation over time.

Thus it is also necessary to ensure connectivity within each subgroup and across subgroups as well as global connectivity.
We call this ability of multi-robot systems to accommodate different behaviors simultaneously within a single connected robot team while maintaining safety (collision avoidance with other robots and possibly obstacles) and within and across subgroup connectivity  \textit{Behavior Mixing}.

\textit{To the best of our knowledge}, there is no existing work on connectivity maintenance that can ensure both global  and subgroup connectivity for behavior mixing. 
Most of the connectivity control methods use either 1) \textit{local methods} that seek to preserve the initial connectivity graph topology over time \cite{dimarogonas2008decentralized, ji2007distributed, zavlanos2007flocking}, or 2) \textit{global methods} that aim to preserve the global algebraic connectivity of the communication graph by deriving secondary connectivity controllers for keeping the second smallest eigenvalue of the graph Lapacian positive at all times \cite{giordano2011bilateral, sabattini2013decentralized, yang2010decentralized, williams2015global}.
Recent work \cite{banfi2018optimal,chandarana2018decentralized, majcherczyk2018decentralized, guerrero2018design, li2017decentralized, aragues2014triggered}
have explored the idea of redeploying a certain number of robots to act as communication relays, while aiming to allow the rest of the robots to perform their original tasks. 
These approaches often has no guarantee that the perturbation from the connectivity controllers is minimum over the original robot task-related controllers. 
On the other hand, recent advances in permissive control barrier functions \cite{ames2019control} have been widely applied to multi-robot systems with guarantee on the forward invariance of desired sets, e.g. robots staying collision-free and connected with predefined communication graph at all times \cite{wang2016multi, li2018formally}, while minimizing revisions to the original robot controllers.
However, these capabilities are achieved by predefining the connectivity constraints so as to preserve \emph{fixed predefined} communication topology \cite{li2018formally, wang2016multi}. It would be desired for the robots to compute in real-time the \emph{optimal} communication topology/connectivity constraints to preserve that provides provably highest freedom for executing the original robot tasks with required network connectivity.

To that end, in this paper we develop a generalized behavior mixing framework with minimum global and subgroup connectivity maintenance. Such framework is based on a bilevel optimization process that 1) incorporates a novel distributed \textit{Minimum Connectivity Constraint Spanning Tree (MCCST)} to compute real-time \emph{minimum connectivity constraints}, and 2) minimizes the revision to the original controllers subject to our invoked connectivity constraints and collision avoidance constraints formulated by the barrier certificates with control barrier functions (CBF) \cite{wang2016multi, ames2019control}. In particular, MCCST computes the provably optimal set of communication links for the robots to maintain, which (a) has minimum number of links, and (b) invokes the connectivity constraints for global and subgroup connectivity least likely to be violated by the original controllers. 
Minimum connectivity maintenance is thus achieved by minimally modifying the original controllers to preserve these dynamic \emph{least constraining} communication links and avoid collisions.

Our paper presents the following contributions: (1) a generalized bilevel optimization based behavior mixing framework to enable \textit{simultaneous execution of different behaviors and sequences of behaviors within a single robot team,} while ensuring global and subgroup connectivity and collision avoidance;
(2) a novel distributed MCCST method with quantified relationship between original task-related controllers and connectivity constraints to efficiently select \textit{real-time minimum behavior mixing connectivity constraints with provably optimality guarantee},
(3) \textit{computationally efficient} construction of MCCST that is \textit{scalable} to  large number of robots and suitable for real-time computation to accommodate dynamic changes in the environment. 

\section{Behavior Mixing}
Consider a robotic team $\mathcal{S}$ consisting of $N$ mobile robots in a planar space, with the position and single integrator dynamics of each robot $i\in\{1,\ldots,N\}$ denoted by $x_i\in \mathbb{R}^2$ and $\dot{x}_i=u_i\in \mathbb{R}^2$ respectively. Each robot can connect and communicate directly with other robots within its spatial proximity. The communication graph of the robotic team is defined as $\mathcal{G} = (\mathcal{V},\mathcal{E})$ where each node $v \in \mathcal{V}$ represents a robot. If the spatial distance between robots $v_i,v_j \in \mathcal{V}$ is less or equal to the communication radius $R_c$ (i.e. $\norm{x_i-x_j}\leq R_c$), then we assume the two can communicate and edge $(v_i,v_j) \in \mathcal{E}$ is undirected (i.e. $(v_i,v_j) \in \mathcal{E} \Leftrightarrow (v_j,v_i) \in \mathcal{E}$).

\subsection{Safety and Connectivity Constraints using Barrier Certificates}
Consider the joint robot states $\mathbf{x}=\{x_1,\ldots,x_N\}\in \mathbb{R}^{2N}$ and define the minimum inter-robot safe distance as $R_s$, for any pair-wise inter-robot collision avoidance constraint between robots $i$ and $j$. We have the following condition defining the safe set of $\mathbf{x}$.
\begin{equation}\label{eq:safesetij}\footnotesize
    \begin{split}
h_{i,j}^s(\mathbf{x})&=\norm{x_i-x_j}^2-R_{s}^2,\qquad \forall i>j \\
\mathcal{H}_{i,j}^s&=\{\mathbf{x}\in \mathbb{R}^{2N}:h^s_{i,j}(\mathbf{x})\geq 0\}
    \end{split}
\end{equation}
The set of $\mathcal{H}_{i,j}^s$ indicates the safety set from which robot $i$ and $j$ will never collide. For the entire robotic team, the safety set can be composed as follows.
\begin{equation}\label{eq:safeset}\footnotesize
    \mathcal{H}^s = \bigcap_{\{v_i,v_j\in \mathcal{V}:i>j\}} \mathcal{H}_{i,j}^s
\end{equation}
\cite{wang2017safety} proposed the safety barrier certificates $\mathcal{B}^s(\mathbf{x})$ using control barrier functions $h^s_{i,j}(\cdot)$ that map the constrained safety set (\ref{eq:safeset}) of $\mathbf{x}$ to the admissible joint control space $\mathbf{u}\in \mathbb{R}^{2N}$ for ensuring $h^s_{i,j}(\cdot)\geq 0$ at all time. The result is summarized as follows.
\begin{equation}\label{eq:safebarrier}\footnotesize
\mathcal{B}^s(\mathbf{x})=\{\mathbf{u}\in \mathbb{R}^{2N}:\dot{h}^s_{i,j}(\mathbf{x})+\gamma h^s_{i,j}(\mathbf{x})\geq 0, \forall i>j\}
\end{equation}
where $\gamma$ is a user-defined parameter to confine the available sets. It is proven in \cite{wang2017safety} that the forward invariance of the safety set $\mathcal{H}^s$ is ensured as long as the joint control input $\mathbf{u}$ stays in set $\mathcal{B}^s(\mathbf{x})$. In other words, the robots will always stay safe if they are initially inter-robot collision free and the control input lies in the set $\mathcal{B}^s(\mathbf{x})$. 
The constrained control space in (\ref{eq:safebarrier}) corresponds to a class of linear constraints over pair-wise control inputs $u_i$ and $u_j$.
%for $\forall i>j$.

Likewise, if the connectivity constraint is enforced between  pair-wise robots $i$ and $j$ to ensure inter-robot distance not larger than communication range $R_c$, we have
\begin{equation}\label{eq:connsetij}\footnotesize
      \begin{split}
h_{i,j}^c(\mathbf{x})&=R_{c}^2-\norm{x_i-x_j}^2\\
\mathcal{H}_{i,j}^c&=\{\mathbf{x}\in \mathbb{R}^{2N}:h^c_{i,j}(\mathbf{x})\geq 0\}
    \end{split}
\end{equation}
The set of $\mathcal{H}^c_{i,j}$ indicates the feasible set on $\mathbf{x}$ from which robot $i$ and $j$ will never lose connectivity. Consider any connectivity spanning graph $\mathcal{G}^c=(\mathcal{V},\mathcal{E}^c)\subseteq\mathcal{G} $ to enforce, the corresponding constrained set can be composed as follows.
\begin{equation}\label{eq:connset}\footnotesize
    \mathcal{H}^c(\mathcal{G}^c) = \bigcap_{\{v_i,v_j\in \mathcal{V}:(v_i,v_j)\in \mathcal{E}^c\}} \mathcal{H}_{i,j}^c
\end{equation}
\vspace{-0.3cm}

Similar to the safety barrier certificates in (\ref{eq:safebarrier}), the connectivity barrier certificates \cite{wang2016multi} are defined as follows indicating another class of linear constraints over pair-wise control inputs $u_i$ and $u_j$ for $(v_i,v_j)\in \mathcal{E}^c$ at any time point $t$.
\begin{equation}\label{eq:connbarrier}\footnotesize
\mathcal{B}^c(\mathbf{x},\mathcal{G}^c)=\{\mathbf{u}\in \mathbb{R}^{2N}:\dot{h}^c_{i,j}(\mathbf{x})+\gamma h^c_{i,j}(\mathbf{x})\geq 0, \forall (v_i,v_j)\in \mathcal{E}^c\}
\end{equation}

\subsection{Bilevel Optimization for Behavior Mixing}

In behavior mixing, we assume the robotic team is tasked with $M$ simultaneous behaviors and has been partitioned into $M$ sub-groups $\mathcal{S}=\{\mathcal{S}_1,\ldots,\mathcal{S}_M\}$, with
each robot $i$ already assigned to a sub-group $\mathcal{S}_m$ and with original task-related controller $u_i=\hat{u}_i$.
To ensure successful behavior mixing, the global connectivity graph $\mathcal{G}$ and the induced subgroup connectivity graph $\mathcal{G}_m =\mathcal{G}[\mathcal{V}_m]\subseteq \mathcal{G}$ where $\mathcal{V}_m\subseteq \mathcal{V}$ containing robots within the same sub-group $\mathcal{S}_m$ for all $m=1,\ldots,M$ should be connected at all time. We assume these connectivity constraints are satisfied initially.
With the defined forms of safety and connectivity constraints in (\ref{eq:safebarrier}) and (\ref{eq:connbarrier}), we formally define the \emph{behavior mixing} problem as a bilevel optimization process at each time step as follows. 

\vspace{-0.5cm}
{\footnotesize
\begin{align}
 &\mathbf{u}^* = \argmin_{\mathcal{G}^c,\mathbf{u}} \sum_{i=1}^{N}\norm{u_i-\Hat{u}_i}^2 \label{eq:rawobj}\\
 \text{s.t.} &\quad \mathcal{G}^c=(\mathcal{V}^c,\mathcal{E}^c)\subseteq \mathcal{G}\quad \text{is connected}\nonumber\\
&\quad  \mathcal{G}_m=\mathcal{G}^c[\mathcal{V}_m]\quad \text{is connected}\quad \forall m=1,\ldots,M \label{eq:rawconn}\\
&\quad \mathbf{u}\in \mathcal{B}^s(\mathbf{x})\bigcap \mathcal{B}^c(\mathbf{x},\mathcal{G}^c),\quad \norm{u_i}\leq \alpha_i,\forall i=1,\ldots,N \label{eq:rawconst}
\end{align} }
\vspace{-0.6cm}

This bilevel optimization problem can be solved by two-steps: find 1) the optimal connectivity spanning graph $\mathcal{G}^{c*}\subseteq \mathcal{G}$ to preserve, and 2) the one-step control inputs $\mathbf{u}^*\in \mathbb{R}^{2N}$ bounded by maximum velocities $\{\alpha_i\}$ and minimally deviated from $\hat{u}_i$ subject to constraints in (\ref{eq:rawconst}) with $\mathcal{G}^{c}=\mathcal{G}^{c*}$.

\section{Behavior Mixing using Distributed Selection of Minimum Connectivity Constraints}
\subsection{Minimum Connectivity Constraint Spanning Tree (MCCST)}
First we consider the sub-problem of selecting optimal connectivity spanning graph $\mathcal{G}^{c*}\subseteq \mathcal{G}$ in Eq. (\ref{eq:rawobj}) that introduces minimum connectivity constraints. 
As each edge $(v_i,v_j)\in \mathcal{E}^c$ in a candidate graph $\mathcal{G}^c$ enforces one constraint between robot $i,j$ in (\ref{eq:connbarrier}), the graph $\mathcal{G}^{c*}$ whose edges define the minimum connectivity constraints must exist among the set of all spanning trees $\mathcal{T}$ of $\mathcal{G}$ that have the minimum number of edges (i.e. $N-1$) for $\mathcal{G}^{c*}$ to stay connected.

Hence, the problem boils down to find the optimal spanning tree $\mathcal{G}^{c*}=\mathcal{T}^{c*}\in \mathcal{T}$ of $\mathcal{G}$ whose edges invoke the minimum connectivity constraints in the form of (\ref{eq:connbarrier}) over the robots' controllers. To quantify the strength of connectivity constraint by an edge $(v_i,v_j)\in \mathcal{E}$, 
%in $\mathcal{T}^c=(\mathcal{V},\mathcal{E}^{T})$, 
we introduce the weight assignment defined as follows.
\begin{equation}\footnotesize
    w_{i,j} = \dot{h}^c_{i,j}(\mathbf{x},\Hat{u}_i,\Hat{u}_j)+\gamma h^c_{i,j}(\mathbf{x}), \forall (v_i,v_j)\in \mathcal{E}
\end{equation}
Compared to the connectivity constraint in (\ref{eq:connbarrier}), $w_{i,j}$ indicates the violation of the pair-wise connectivity constraint between the two robots under the original controllers $\hat{u}_i,\hat{u}_j$, with the higher value of $w_{i,j}$ the less violated the connectivity constraint is. This quantifies how likely the existing connectivity link is going to break if no revision made to the controller. It is desired to preserve those links with larger $w_{i,j}$ implying smaller revision needed for the controllers to keep the links connected. With that, each candidate spanning tree $\mathcal{T}^c\in \mathcal{G}$ is redefined as a weighted spanning tree $\mathcal{T}_w^c=(\mathcal{V},\mathcal{E}^T,\mathcal{W}^T)$ with $\mathcal{W}^T=\{-w_{i,j}\}$. Hence the optimal connectivity graph $\mathcal{G}^{c*}$ with constraints in (\ref{eq:rawconn}) can be obtained as follows.
\begin{equation}\footnotesize
\begin{split}
   &\mathcal{G}^{c*} = \argmax_{\mathcal{T}_w^c\in \mathcal{T}}\sum_{(v_i,v_j)\in \mathcal{E}^{T}}w_{i,j}= \argmin_{\mathcal{T}_w^c\in \mathcal{T}} \sum_{(v_i,v_j)\in \mathcal{E}^{T}}-w_{i,j}\\
 \text{s.t.} &\quad  \mathcal{T}_m=\mathcal{T}_w^c[\mathcal{V}_m]\quad \text{is connected}\quad \forall m=1,\ldots,M \label{eq:expobj}  
\end{split}
\end{equation}
The optimal solution of (\ref{eq:expobj}) is the Minimal Spanning Tree (MST) weighted by $\{-w_{i,j}\}$ and constrained by sub-group connectivity requirements. We propose to define another class of spanning trees as follows and relate its unconstrained MST to the solution of the constrained MST in (\ref{eq:expobj}).
\begin{definition}\label{def:ccst}
Given a connectivity graph $\mathcal{G}$
% weighted spanning tree $\mathcal{T}_w^c=(\mathcal{V},\mathcal{E}^T,\mathcal{W}^T)$ 
and for all edges $(v_i,v_j)\in \mathcal{E}$ on $\mathcal{G}$, redefine their weights by the following.
\begin{align}\footnotesize
\label{eq:neww}
w'_{i,j} = \left\{ \begin{gathered}
  \lambda\cdot w_{i,j}, \quad \text{if} \quad \text{$v_i$ and $v_j$ are in the same sub-group } \hfill \\
  w_{i,j}, \quad \text{if} \quad \text{$v_i$ and $v_j$ are in different sub-groups} \hfill \\
\end{gathered}  \right.
\end{align}
where $\lambda\in\{\lambda \gg 1: \lambda \cdot w_{i,j}\gg w_{i',j'},\forall v_i,v_i',v_j,v_j' \in \mathcal{V}\}$ is a unique user-defined constant for the entire graph $\mathcal
{G}$. The weight-modified graph is denoted as $\mathcal{G}'$. Then we call the redefined spanning tree $\mathcal{T}_w^{c'}=(\mathcal{V},\mathcal{E}^T,\mathcal{W}^{T'})$ as the Connectivity Constraint Spanning Tree (CCST).
\end{definition} 

The Definition \ref{def:ccst} introduces a new class of spanning trees (CCST) $\mathcal{T}_w^{c'}$ equivalent to the original spanning trees $\mathcal{T}_w^{c}$ with inflated weights over the edges connecting robots in the same sub-group. In particular, the designed parameter $\lambda$ in (\ref{eq:neww}) ensures that after inflation the new weights $-w'_{i,j}$ over edges connecting different subgroups are always larger than any edges within all the subgroups for $\mathcal{T}_w^{c'}$. As we will prove by the following Lemma \ref{lemma:subgroup} and Theorem \ref{theorem:mccst}, this guarantees that the computed MST $\mathcal{T}_w^{c'}$ becomes the solution of constrained MST $\mathcal{T}_w^c$ in (\ref{eq:expobj}), namely, the MST $\mathcal{T}_w^{c'}$ contains the MST of each subgroup as well, ensuring that the subgroups are also connected in an optimal way. We review some useful definitions in graph theory \cite{gallager1983distributed}:
\begin{itemize}
\item \textit{fragment}: a subtree of Minimum Spanning Tree;
\item \textit{outgoing edge}: a edge of a fragment if one adjacent node is in the fragment and the other is not.
\end{itemize}
The first definition describes that a connected set of nodes and edges of the MST is called a fragment. By this definition, a single node is also a fragment by itself. In the following discussion, we focus on \textit{minimum-weight outgoing edge (MWOE)}, which is the edge with minimum weight among all outgoing edge of a fragment. 

\begin{lemma} \label{lemma:frag}
Let $e_{min}$ be a minimum-weight
outgoing edge (MWOE) of a fragment. Connecting $e_{min}$ and its adjacent node in a different fragment yields another fragment in MST.
\end{lemma}
The proof of Lemma \ref{lemma:frag} can be found in both \cite{gallager1983distributed} and \cite{peleg2000distributed}.

With Lemma \ref{lemma:frag}, the process of constructing MST is as follows \cite{gallager1983distributed}:
\begin{itemize}
\item Each node starts as a fragment by itself
\item Each fragment iteratively connects with MWOE fragment
\end{itemize}
This process will result in the MST of the given graph.

\begin{lemma} \label{lemma:subgroup}
By following the process above on $\mathcal{G}'$ in Definition \ref{def:ccst}, all nodes within the same sub-group will form a MST fragment before connecting to other sub-group.
\end{lemma}
\begin{proof}
We prove by contradiction. Suppose the node $v_i$ from sub-group graph $\mathcal{G}_i'$ connects with node $v_j$ first, which belongs to sub-group graph $\mathcal{G}_j'$, $i \neq j$. 
From the MST construction process we know that, at each iteration, the edge added is the minimum-weight outgoing edge of the connecting fragment. In this case, the weight $w_{i, j}$ of the edge between $v_i$ and $v_j$ is the minimum of all outgoing edges of $v_i$.
Let $v_{i'} \in \mathcal{G}_i'$ where there exists an outgoing edge between $v_i$ and  $v_{i'}$, then we know that the weight $w_{i, j}' < w_{i, i'}'$. This contradicts with the property of $\mathcal{G}'$ in Equation \ref{eq:neww}.
\end{proof}

\begin{theorem}\label{theorem:mccst}
Given the redefined Connectivity Constraint Spanning Tree (CCST) $\mathcal{T}_w^{c'}=(\mathcal{V},\mathcal{E}^T,\mathcal{W}^{T'})$ in Definition \ref{def:ccst} and denote minimum weight CCST as $\bar{\mathcal{T}}_w^{c'}=\argmin_{\mathcal{T}_w^{c'}\in \mathcal{T}} \sum_{(v_i,v_j)\in \mathcal{E}^{T}}-w'_{i,j}$, we have: $\bar{\mathcal{T}}_w^{c'} = \mathcal{G}^{c*}$ in Equation (\ref{eq:expobj}). Namely, the Minimum Spanning Tree $\bar{\mathcal{T}}_w^{c'}$ of $\mathcal{G}'$ is the optimal solution of $\mathcal{G}^{c*}$ in (\ref{eq:expobj}) and we call the graph $\bar{\mathcal{T}}_w^{c'}$ as Minimum Connectivity Constraint Spanning Tree (MCCST) of the original connected graph $\mathcal{G}$.
\end{theorem}
\begin{proof}
From Lemma \ref{lemma:subgroup}, edges \textit{between} sub-groups will be connected only when edges \textit{within} each sub-groups are connected.
By definition, the MST of graph $\mathcal{G}_i'$ within subgroup $\mathcal{S}_i$ is optimal with minimum total weight, which means 
\begin{equation} \footnotesize
\begin{split}
    \bar{{\mathcal{T}}_w^{c'}}(i)&=\argmin_{\mathcal{T}_w^{c'}(i)\in \mathcal{T}(i)} \sum_{(v_i,v_j)\in \mathcal{E}^{T(i)}}-w'_{i,j}\\
    &=\argmin_{\mathcal{T}_w^{c'}(i)\in \mathcal{T}(i)} \lambda\cdot\sum_{(v_i,v_j)\in \mathcal{E}^{T(i)}}-w_{i,j}\\
    &=\argmin_{\mathcal{T}_w^{c'}(i)\in \mathcal{T}(i)} \sum_{(v_i,v_j)\in \mathcal{E}^{T(i)}}-w_{i,j}
\end{split}
\end{equation}
The equality holds since $\lambda > 0$. Then we consider $v_i$ and $v_j$ in different subgroups, i.e. $\mathcal{S}(v_i) \neq \mathcal{S}(v_j)$, while $(v_i, v_j)$ is the edge in spanning tree edges $\mathcal{E}^{T(i)}$ connecting two subgroups. Then for the next step, connecting the minimum-weighted outgoing edge between different sub-groups, yields
\begin{equation} \footnotesize
\begin{split}
    \bar{{\mathcal{T}}_w^{c'}}
    &=\argmin_{\mathcal{T}_w^{c'}\in \mathcal{T}} \sum_{(v_i,v_j)\in \mathcal{E}^{T(i)}}- w'_{i,j}, \quad \mathcal{S}(v_i) \neq \mathcal{S}(v_j)\\
    &=\argmin_{\mathcal{T}_w^{c'}\in \mathcal{T}} \sum_{(v_i,v_j)\in \mathcal{E}^{T}}- w_{i,j}
\end{split}
\end{equation}
With the same form as in (\ref{eq:expobj}), this concludes the proof.
\end{proof}

In this way, we relax the constrained MST optimization problem in (\ref{eq:expobj}) into unconstrained MST problem with the same optimality guarantee. The connectivity constraints from the obtained MCCST $\bar{\mathcal{T}}_w^{c'}$ are thus minimally violated by the current task-related controllers, implying the least restriction due to global and subgroup connectivity requirements. Such MCCST $\bar{\mathcal{T}}_w^{c'}$ therefore specifies the optimal connectivity graph $\mathcal{G}^{c*}\subseteq \mathcal{G}$ to enforce for behavior mixing in (\ref{eq:rawconn}). Next, we will present a distributed method for computing MCCST.

\subsection{Construction of Distributed Minimum Connectivity Constraint Spanning Tree (MCCST)}
Here we propose a distributed construction of MCCST of $\mathcal{G}$. 
For our problem setting, the topology and weights could change over time, thus a time-optimal real-time algorithm is needed. We develop our algorithm based on the work from \cite{gallager1983distributed,  peleg2000distributed, pandurangan2017time}, but reduce the computation time while sacrificing message optimality. Different from most of the network algorithms such as \cite{pandurangan2017time, peleg2000distributed}, our algorithm does not require synchronization, which also reduces the total time. 
Note that MST is unique for a graph with unique edge weights. Therefore the result is the same from centralized and decentralized construction.

A detailed description of the algorithm is as follows:
\subsubsection{Overview}

Given a graph $\mathcal{G}'=\{\mathcal{V}, \mathcal{E}\}$ with weights defined in Definition \ref{def:ccst} where $|\mathcal{V}|=N$ is the number of robots, the initial state of the system is a singleton graph where each vertex is an individual isolated node without any outgoing edge, and each node is given a distinct id. This gives $N$ fragments and each consists of one vertex. Then each fragment finds the \textit{minimum-weight outgoing edge (MWOE)} and connect with neighboring fragments. Iteratively, the forest of fragments will join as a spanning tree connecting all vertices of the graph, resulting as the MCCST.
\begin{algorithm}[t]\footnotesize
    \caption{Distributed MCCST Construction}
    \label{alg:dis_mst}
    \begin{algorithmic}[1]
    \small
    \Input{$a$: adjacency edge weight list of the original weighted graph}
    \Output{edge list of MCCST}
    \Function{ConstructDistributedMCCST}{$A$}
    \State $A$ $\gets$ empty adjacency matrix
    \State $A$ $\gets$ updated from input $a$
    \While {$msg$ $\gets$ getNewMessage($msg\_pool$)}
    \If {not initialized}
    \State $A$ $\gets$ initialRound($msg$, $A$)
    \Else
    \State $A$ $\gets$ processRound($msg$, $A$)
    \EndIf
    \If{isConnected($A$)}
    \State \Return{getEdgeList($A$)}
    \EndIf
    \EndWhile
    \If{isEmpty($msg\_pool$)}
    \State resetRound()
    \EndIf
    \EndFunction
    \end{algorithmic}
\end{algorithm}
\setlength{\textfloatsep}{0pt}

As shown in Algorithm \ref{alg:dis_mst}, each robot takes an input of neighboring edge weights and connectivity information, then outputs the computed MCCST edge list. The incoming message is processed according to whether the node is being initialized or not. The process will reset when there is no new message in the message pool, which implies all the fragments finish updating within themselves and new MWOE need to be connected and a new round begins.
%\vspace{-0.2cm}
\subsubsection{Initial Round}
\begin{algorithm}[h]\footnotesize
\caption{Initial Round of MST Construction}
\label{alg:init_mst}
    \begin{algorithmic}[1]
    \Input{$msg$: incoming messages, $A$: current adjacency matrix}
    \Output{$A$: updated adjacency matrix}
    \Function{initialRound}{$msg$, $A$}
    \State connect with neighbor with $MWOE$
    \State $leader\_id$ $\gets$ min($self\_id$, $neighbor\_id$)
    \State $A$ $\gets$ update with $msg$
    \If{ no new information in $msg$}
    \State finish initial round
    \EndIf
    \State send $init$ message to $msg\_pool$
    \State \Return{$A$}
    \EndFunction
    \end{algorithmic}
\end{algorithm}
%\vspace{-0.2cm}
In the initial round, as shown in Algorithm \ref{alg:init_mst}, each fragment initially only contains one vertex. Each node directly connect with the neighbor with MWOE. However, to keep information within a fragment consistent and avoid additional computation, the node with the smallest id is selected as fragment leader. Information keeps updating within the fragment until every node has the same adjacency matrix of its fragment tree.

\begin{figure*}[!htbp]
%\captionsetup{skip=0pt}
  \centering
  \begin{subfigure}{0.3\textwidth}%32
\includegraphics[width=\textwidth]{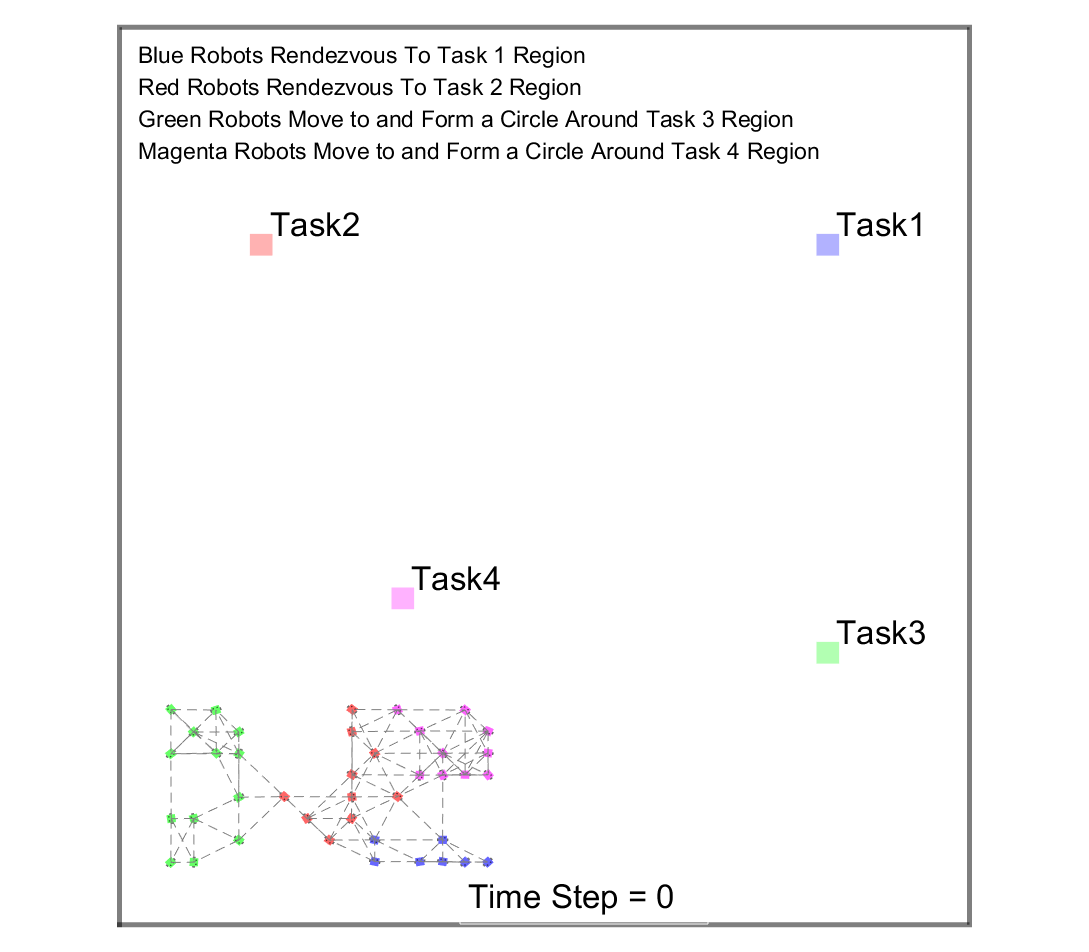}
    \caption{Time Step $=0$}
    \label{fig:init}
  \end{subfigure}
  \begin{subfigure}{0.3\textwidth}
\includegraphics[width=\textwidth]{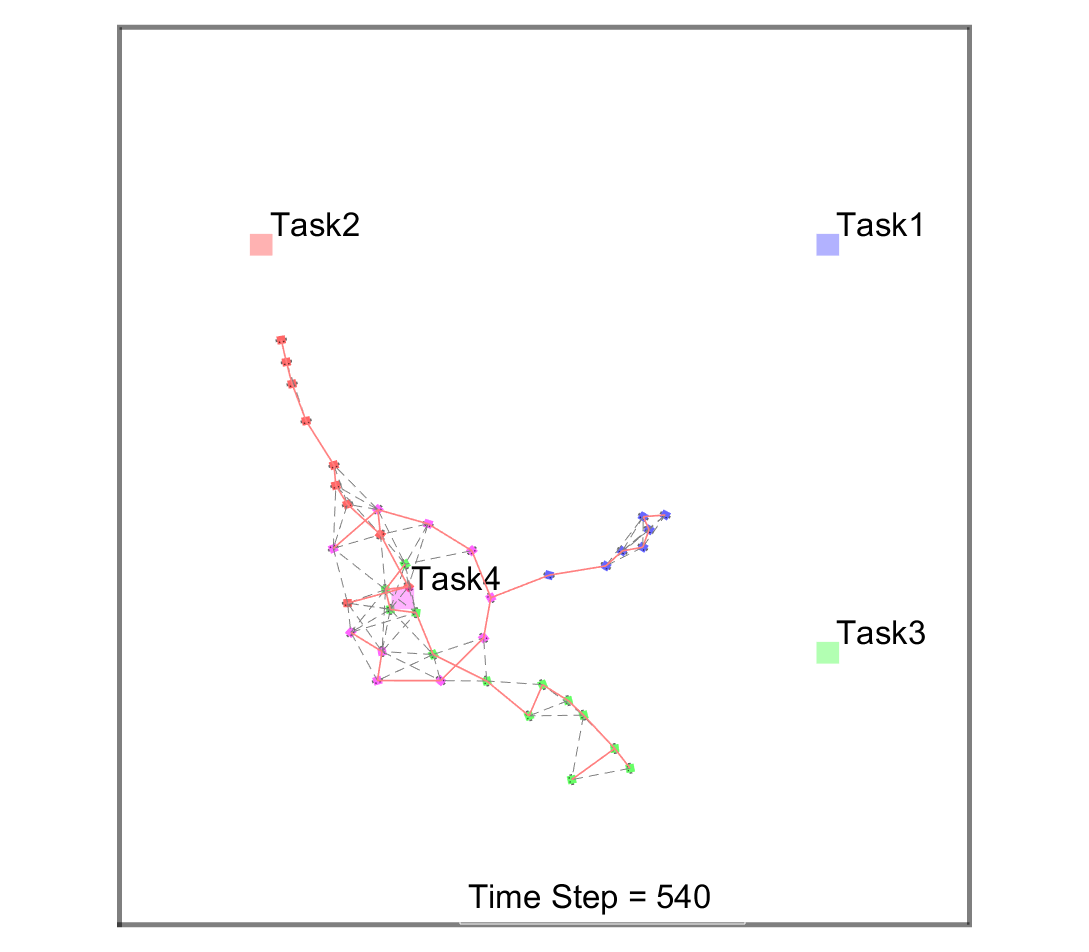}
    \caption{Time Step $=540$ (MCCST)}
    \label{fig:middle}
  \end{subfigure}
  \begin{subfigure}{0.3\textwidth}
\includegraphics[width=\textwidth]{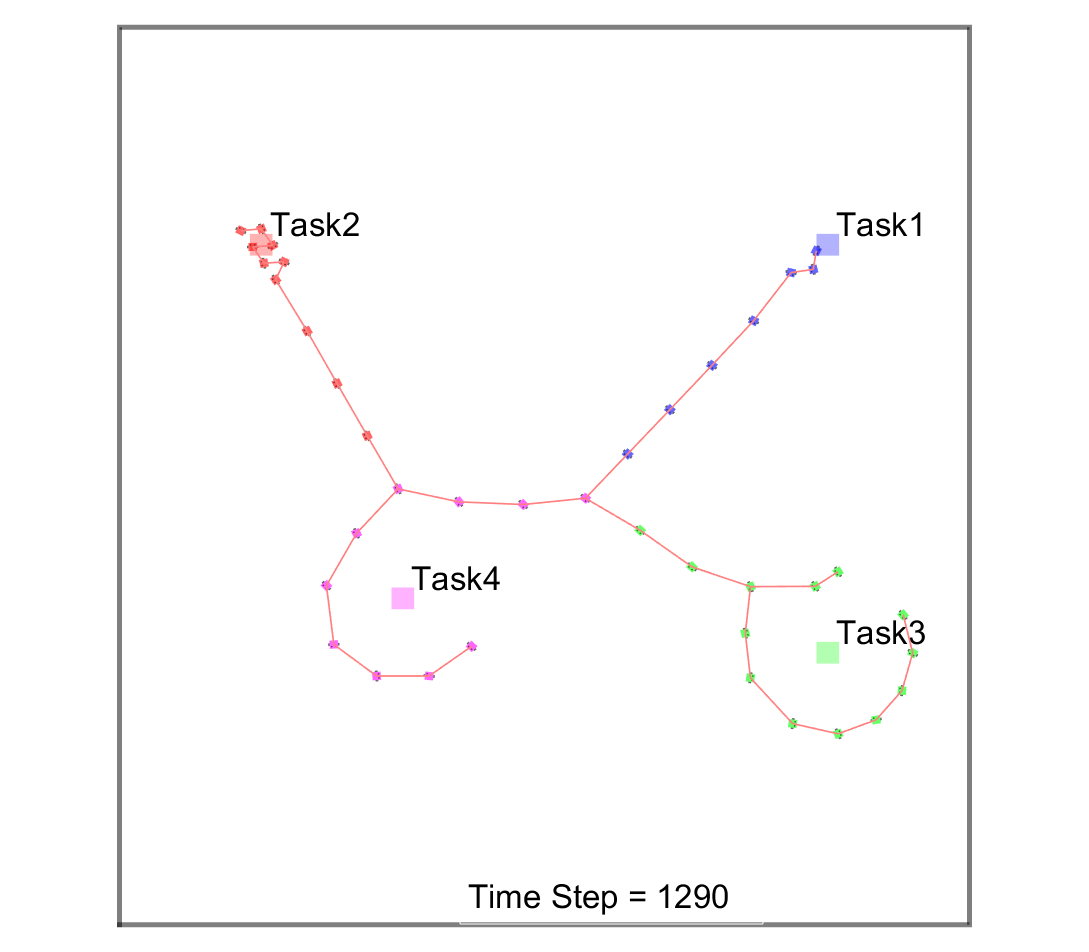}
    \caption{Time Step $=1290$ (MCCST, Converged)}
    \label{fig:finalours}
  \end{subfigure}
    \begin{subfigure}{0.3\textwidth}
\includegraphics[width=\textwidth]{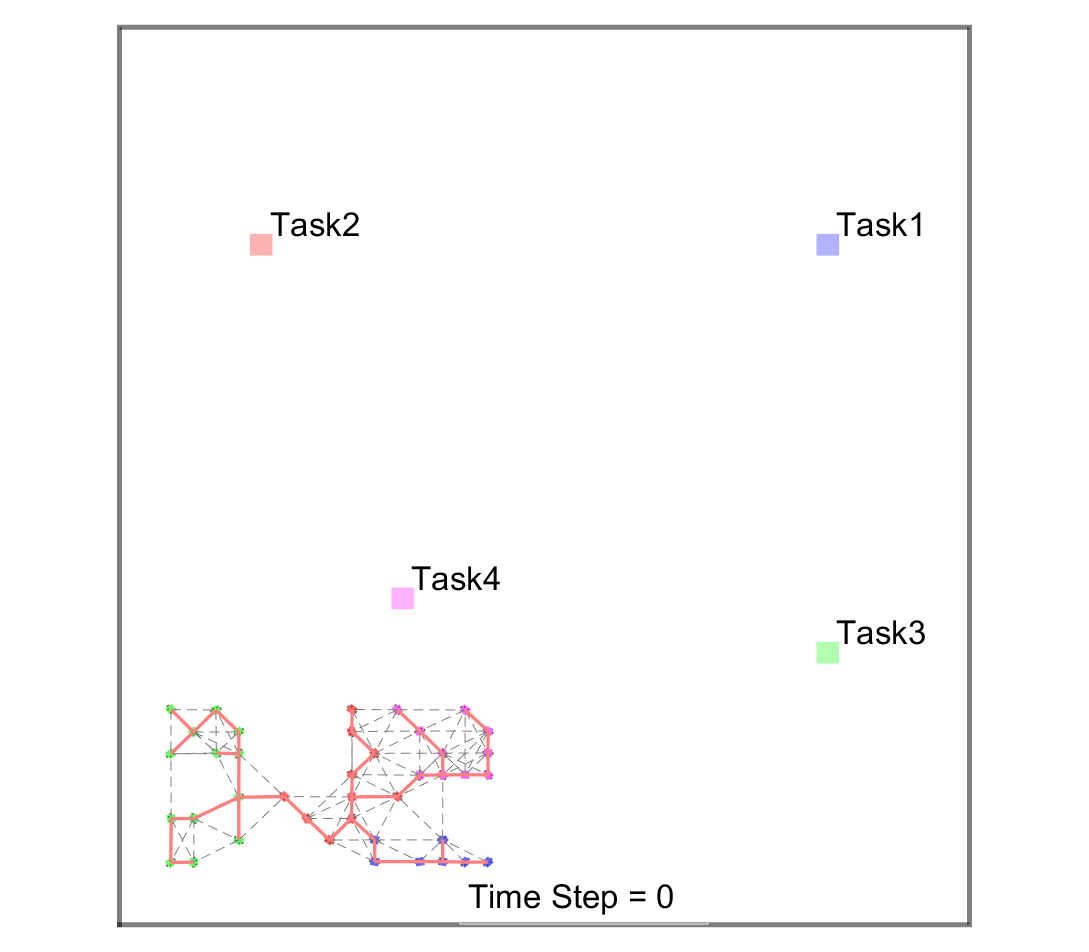}
    \caption{Time Step $=0$}
    \label{fig:initspan}
  \end{subfigure}
  \begin{subfigure}{0.3\textwidth}
\includegraphics[width=\textwidth]{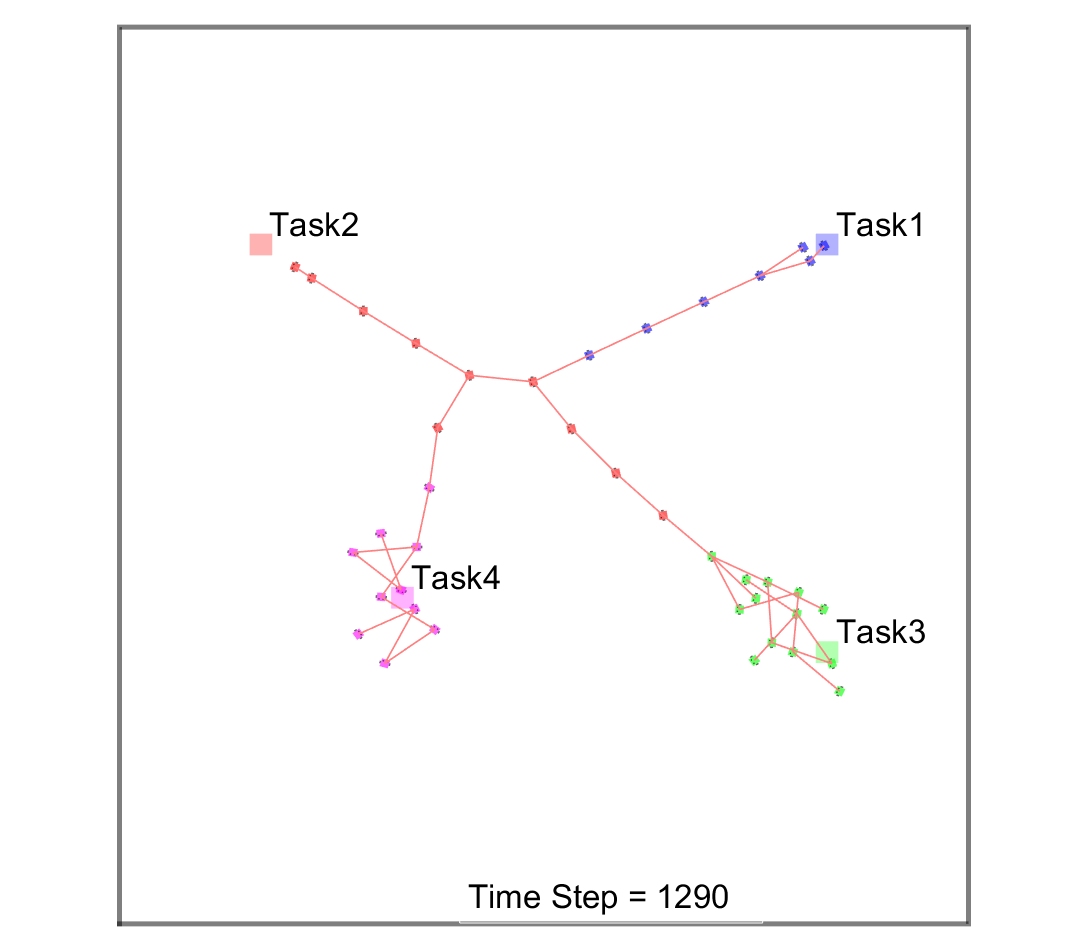}
    \caption{Time Step $=1290$ (Fixed Initial MST, Converged)}
    \label{fig:finalinitmst}
  \end{subfigure}
  \begin{subfigure}{0.3\textwidth}
\includegraphics[width=\textwidth]{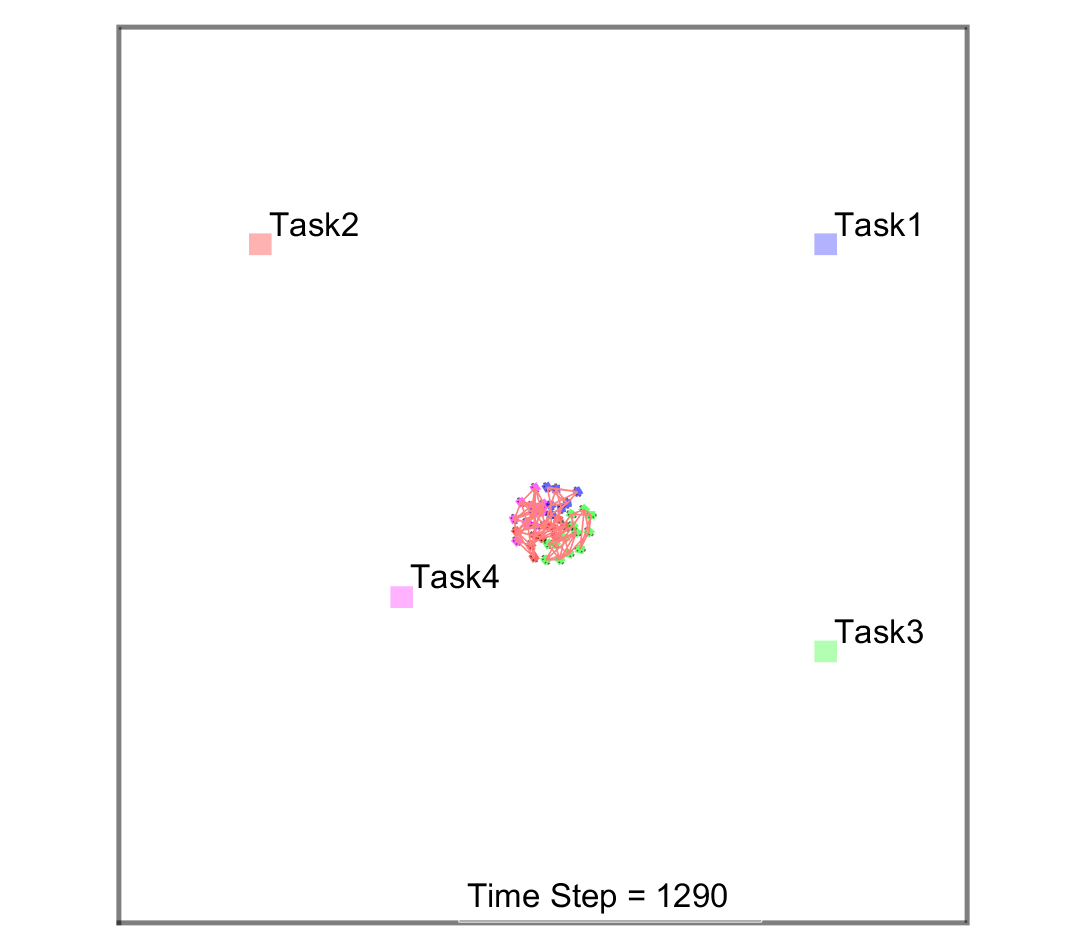}
    \caption{Time Step $=1290$ (Fixed Initial Connectivity, Converged)}
    \label{fig:finalinitconn}
  \end{subfigure}
\caption{Simulation example of 40 robots tasked to four different places with behaviour mixing: blue robots and red robots rendezvous to regions of blue task 1 and red task 2 respectively, while green robots and magenta robots move to region of green task 3 and magenta task 4 and form a circle around the regions. Grey dashed lines in (a),(b),(d) denote current connectivity edges and red lines in (a)-(f) denote current active connectivity graph invoking pair-wise connectivity constraints. Compared to fixed inter-robot connectivity constraints from initial MST (e) and initial connectivity graph (f), our proposed MCCST approach (c) enables minimally perturbed task performance due to invoked minimum connectivity constraints on the robots.
}
  \label{fig:sim1}
  \vspace{-12pt}
\end{figure*}

% \begin{figure*}[!htbp]
\begin{figure*}
\captionsetup{skip=0pt}
  \centering
  \begin{subfigure}{0.24\textwidth}%24
\includegraphics[width=\textwidth]{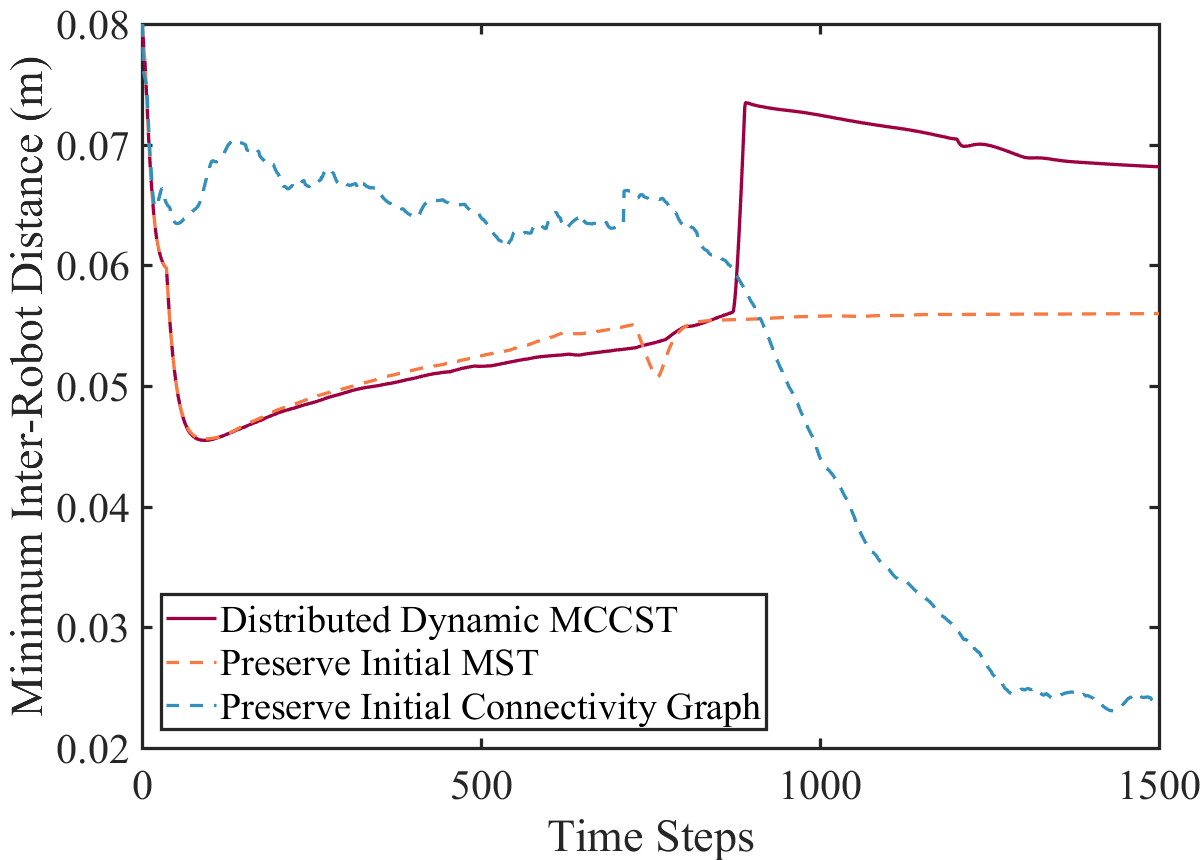}
    \caption{Minimum inter-robot distance}
    \label{fig1:interdistance}
  \end{subfigure}
  \begin{subfigure}{0.24\textwidth}
\includegraphics[width=\textwidth]{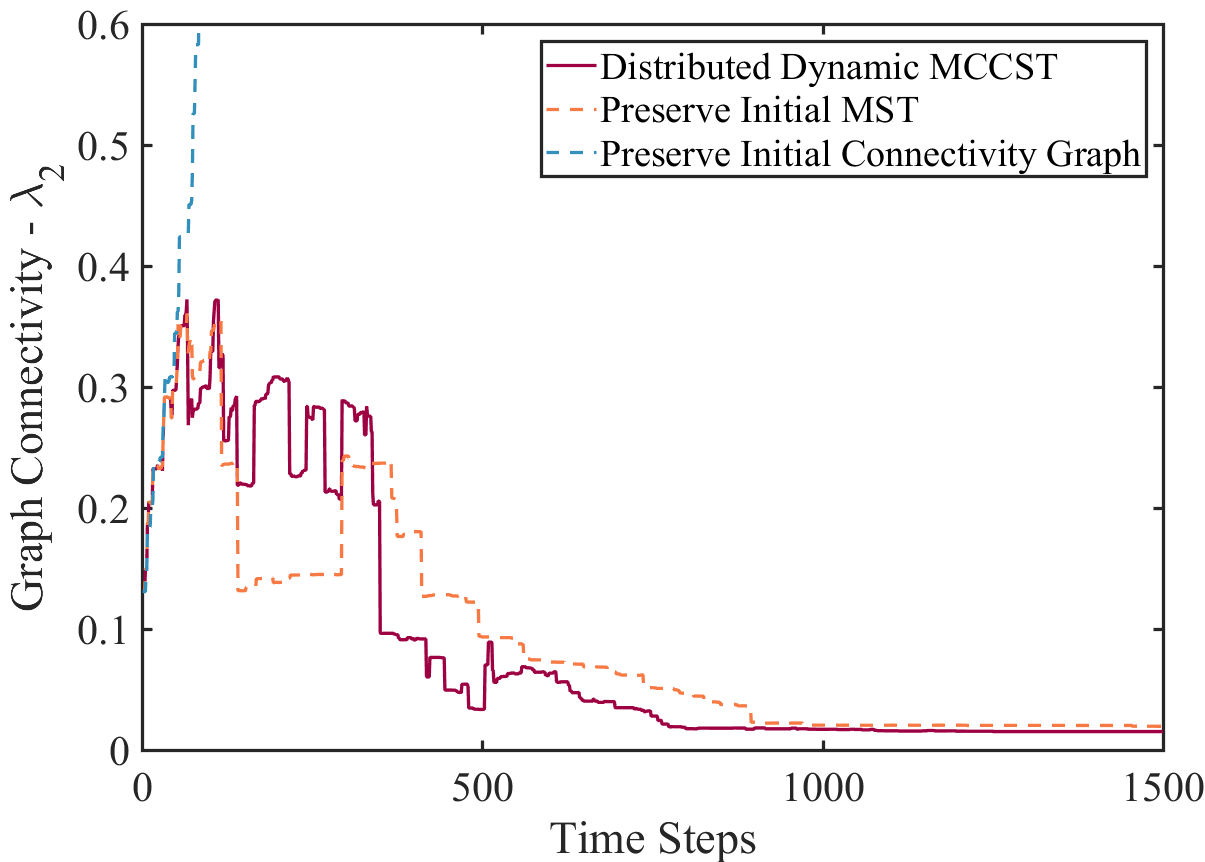}
    \caption{Algebraic connectivity}
    \label{fig1:conn}
  \end{subfigure}
  \begin{subfigure}{0.24\textwidth}
\includegraphics[width=\textwidth]{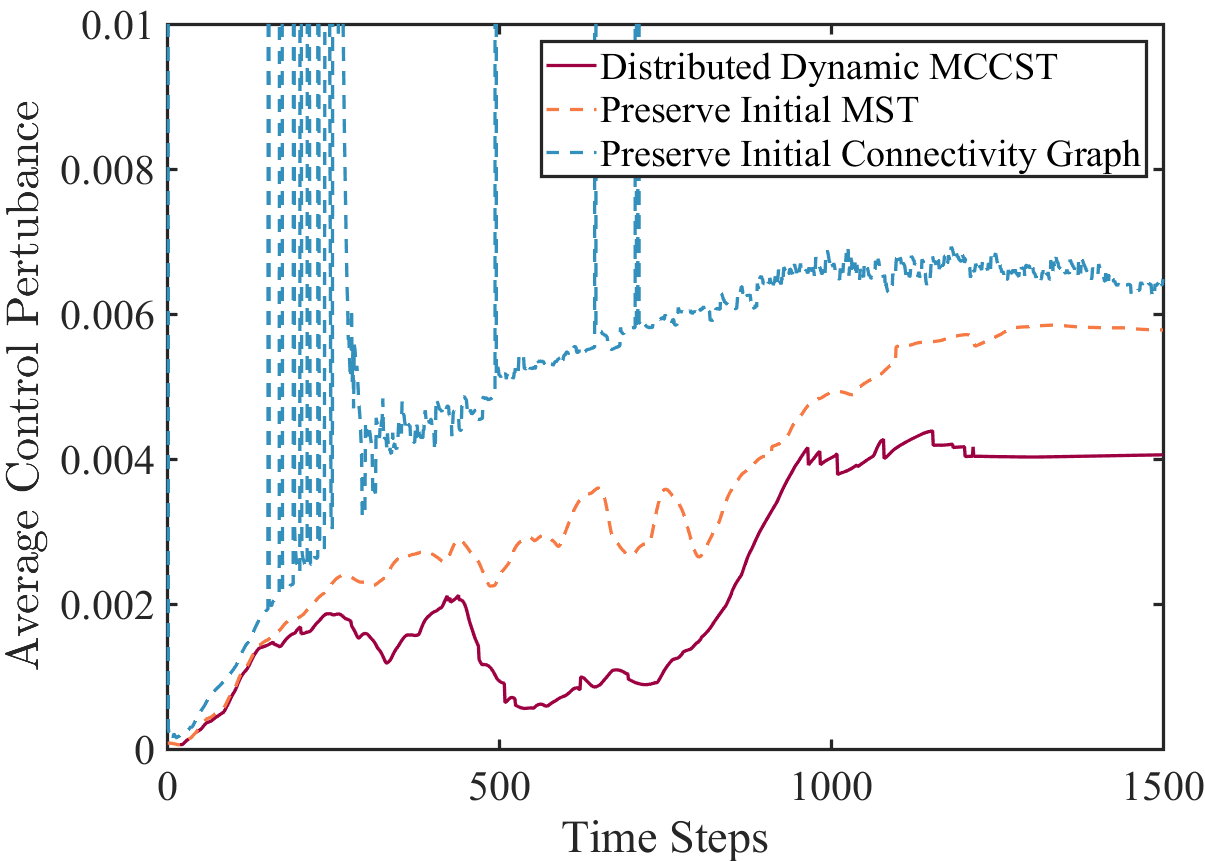}
    \caption{Average control perturbation}
    \label{fig1:ctrl}
  \end{subfigure}
  \begin{subfigure}{0.24\textwidth}
\includegraphics[width=\textwidth]{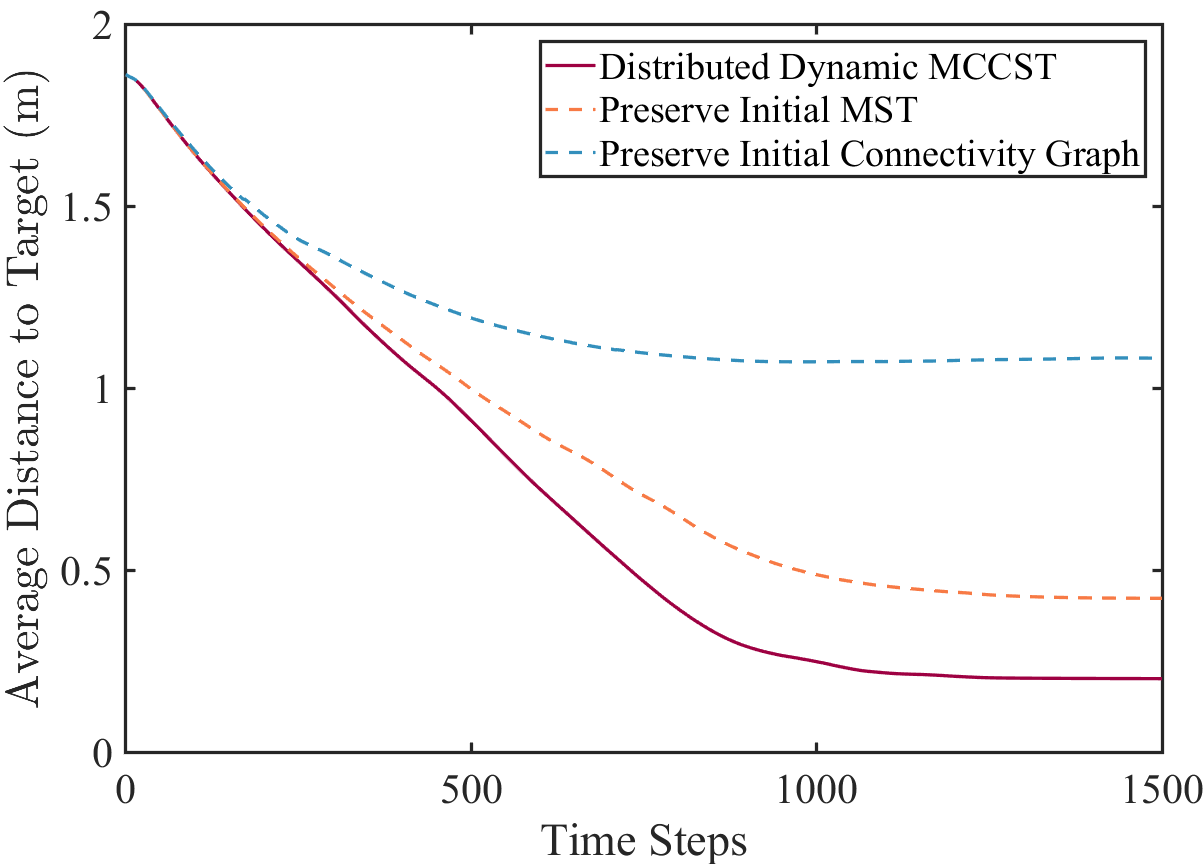}
    \caption{Average distance to target region}
    \label{fig1:perf}
  \end{subfigure}
\caption{Performance comparison of simulation example in Figure~\ref{fig:sim1} w.r.t. different metrics: (a) Minimum inter-robot distance (safety distance is 0.02m), (b) Algebraic connectivity evaluated by second smallest eigenvalue of mutli-robot laplacian matrix. Positive meaning connectivity ensured, (c) Control perturbation computed by $\frac{1}{N}\sum_{i=1}^N\norm{u_i^*-\Hat{u}_i}^2$, (d) Average distance between robots to tasked region (the smaller the better).
}
  \label{fig:sim_setup}
  \vspace{-10pt}
\end{figure*}

\subsubsection{Processing Round}

% \vspace*{-.3cm}
\begin{algorithm}\footnotesize
\caption{Processing Round of MCCST Construction}
\label{alg:proc_mst}
    \begin{algorithmic}[1]
    \small
    \Input{$msg$: incoming messages, $A$: current adjacency matrix}
    \Output{$A$: updated adjacency matrix}
    \Function{processRound}{$msg$, $A$}
    \If{$leader\_id$ is $self\_id$}
    \State $MWOE\_cache\_list$ $\gets$ wait($fragment\_node$)
    \State $MWOE$ $\gets$ min($MWOE\_cache\_list$)
    \If{all $fragment\_node$ reported}
    \State inform the one with $MWOE$ to connect
    \EndIf
    \EndIf
    \If{not reported to leader}
        \If{no MWOE info}
        \State $msg\_pool$[$MWOE$] $\gets$ $get\_info$ message 
        \EndIf
        \If{get information from MWOE neighbor}
        \State $msg\_pool$[$leader\_id$] $\gets$ $report$ message 
        \EndIf
    \EndIf
    \State $A$ $\gets$ update with $msg$
    \State \Return{$A$}
    \EndFunction
    \end{algorithmic}
\end{algorithm}

At each processing round, the fragment leader will determine the minimum-weight outgoing edge (MWOE) in its fragment after receiving all MWOE information from each fragment node (including itself). Since each node in the fragment only has the local knowledge within its own fragment, it will ask the MWOE neighbor for their fragment information, i.e. adjacency matrix, leader id. Whenever a node receives a request to give information, it will reply accordingly. Once each node receives information from MWOE neighbor, it will report to the fragment leader. All connect requests will be accepted and this, by lemma \ref{lemma:frag}, always yields a fragment. When a new connection is made, the two fragments will combine their information and update all the nodes within the fragment with the new information. Iteratively, the construction process will end when every node receives the same updated adjacency matrix representing the MST of the graph. Since only the leader of each fragment updates the adjacency matrix within the fragment, eventually when the algorithm terminates, there will be only one fragment, i.e. the MST, with one leader, marking the convergence of the distributed algorithm.
The convergence speed of our distributed MCCST algorithm is dependent on the topology of the original communication graph and edge weights, ranging from one iteration to $O(\log N)$ iterations with a worst case time complexity of $O(N\log N)$.

Once the final MCCST is obtained as the optimal connectivity graph $\mathcal{G}^{c}=\mathcal{G}^{c*}$ in (\ref{eq:rawconst}), we can specify the safety and connectivity barrier certificates (\ref{eq:safebarrier}) and (\ref{eq:connbarrier}) to invoke a set of linear constraints. Thus the original quadratic programming (QP) problem in (\ref{eq:rawobj}) could be efficiently solved to get optimal revised robot controllers satisfying safety and global and subgroup connectivity constraints for behavior mixing.

\begin{figure*}[!htbp]
\centering
  \begin{subfigure}[t]{0.24\textwidth}%24
    \includegraphics[width=\textwidth]{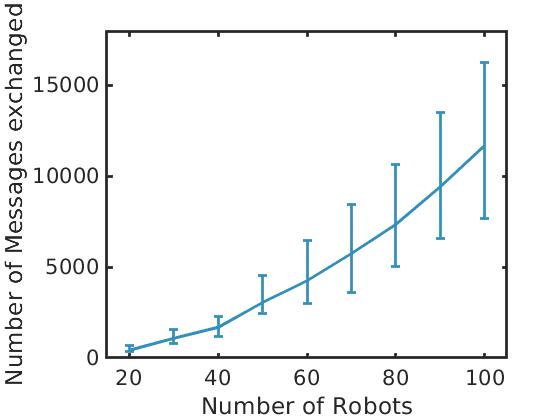}
    \caption{Number of messages exchanged}
    \label{fig:avg_msg_mst}
  \end{subfigure}\hfill
  \begin{subfigure}[t]{0.24\textwidth}
    \includegraphics[width=\textwidth]{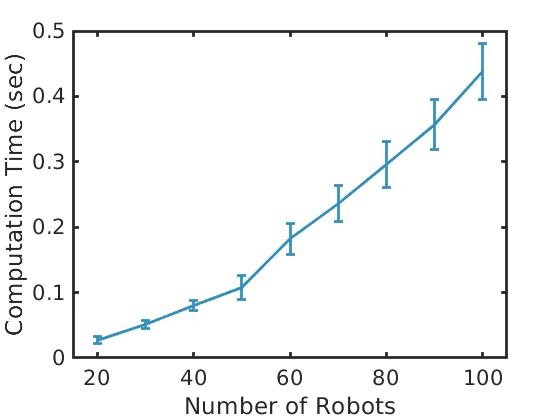}
    \caption{Computation Time}
    \label{fig:avg_time_mst}
  \end{subfigure}\hfill
  \begin{subfigure}[t]{0.24\textwidth}
    \includegraphics[width=\textwidth]{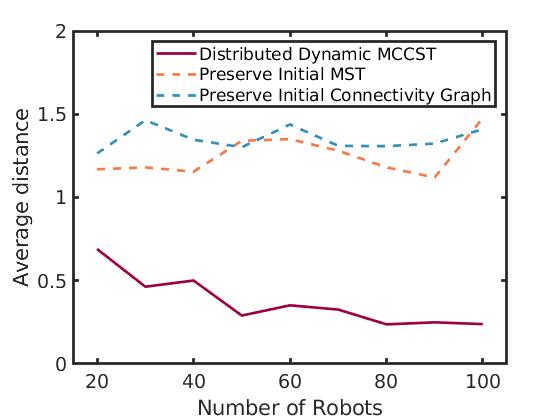}
    \caption{Average Distance to targets}
    \label{fig:avg_distance}
  \end{subfigure}\hfill
  \begin{subfigure}[t]{0.24\textwidth}
    \includegraphics[width=\textwidth]{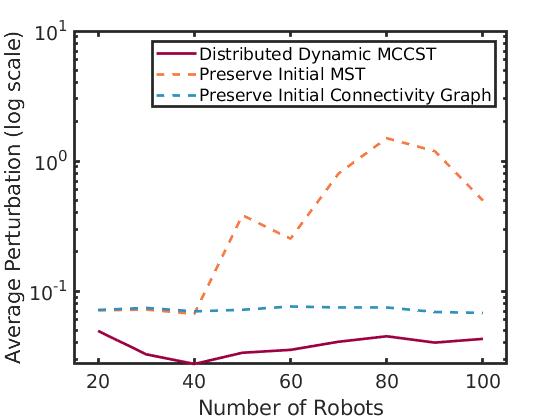}
    \caption{Average Control Perturbation}
    \label{fig:perturbation}
    \vspace{-20pt}
  \end{subfigure}
\caption{Quantitative results summary. (a)-(b) are results from our proposed Distributed MCCST approach. (c)-(d) are comparison results with ours and the other two approaches with static connectivity graph but the same controller (\ref{eq:rawobj}). (a) Number of messages exchanged during the distribute MCCST construction. The error bar shows the maximum and minimum number of messages exchanged. (b) Computation time of constructing the distributed MCCST. The error bar shows the standard deviation. (c) Average distance from robot to target location after converged. (d) Average control perturbation.}
  \label{fig:result_robot_num}
  \vspace{-10pt}
\end{figure*}

\section{Results}

\subsection{Simulation Example}

The first set of experiments are performed on a team of $N=40$ mobile robots with unicycle dynamics as shown in Figure~\ref{fig:sim1}. 
The robot team is divided into $M=4$ subgroups with different colors and is tasked with 4 parallel behaviors. In the figures, robots in blue subgroup 1 and red subgroup 2 execute biased rendezvous behaviors towards the blue task site 1 and red task site 2 respectively, while robots in green subgroup 3 and magenta subgroup 4 perform circle formation behaviors around the green task site 3 and magenta task site 4 respectively. 
For our MCCST method, we apply the minimally revised controllers from (\ref{eq:rawobj}) with single-integrator dynamics to the robots with unicycle dynamics using kinematics mapping from \cite{wang2017safety}.
As shown in Figure~\ref{fig:sim1}a-c, our distributed MCCST approach is able to generate real-time minimum connectivity graph (red edges) from the present connectivity graph (grey edges) so that the invoked connectivity constraints are minimally restrictive to the original behavior controllers. Most of the target behavior configurations have been accomplished as shown in Figure~\ref{fig:sim1}c. The communication relays connecting different subgroups are implicitly formed to provide greater flexibility for the rest of the robots without the need of explicit robots roles assignment as done in \cite{banfi2018optimal, majcherczyk2018decentralized}. This is because our algorithm enforces provably minimum connectivity graph that is least restrictive to the robots. 

In comparison, we present converged results of other two methods with static connectivity graph in Figure~\ref{fig:finalinitmst} and Figure~\ref{fig:finalinitconn} respectively: i) always preserving communication edges in the initial MST (red) depicted in Figure~\ref{fig:sim1}d, and ii) always preserving edges in initial connectivity graph (grey) in Figure~\ref{fig:init} as done in \cite{zavlanos2007flocking}. Since the invoked connectivity graph is fixed as the robots move, they can hardly achieve circle formation (Figure~\ref{fig:sim1}e) or could fall into deadlock (Figure~\ref{fig:sim1}f). Numerical results are provided in Figure~\ref{fig:sim_setup} showing our method ensures safety and connectivity, while having minimal control perturbation due to connectivity and maximum task performance (very close to designated target area as shown from Figure~\ref{fig:sim_setup}d). Note that in Figure~\ref{fig:finalinitmst} the provided comparison method of preserving initial MST from our MCCST without updating in real-time is already better than other barrier certificate based connectivity controllers \cite{li2018formally, wang2017safety} that impose predefined fixed connectivity graph not necessarily as optimal for the tasks.

\subsection{Quantitative Results}

For validating the computation efficiency and scalabiltiy of our algorithm, we run experiments with up to 100 robots and 4 parallel behaviors (four robots subgroups simultaneously rendezvous to four different places with safety and connectivity constraints).
For Figure \ref{fig:avg_msg_mst} and \ref{fig:avg_time_mst}, the experiment is done by computing the distributed MCCST 1500 to 3000 times, depending on the iterations for the system to converge, which varies with the number of robots and graph topology. The complexity of the worst case for both message and time is $O(N \log N)$. However, the average case, as shown in the figure, is better than $O(N \log N)$. Figure \ref{fig:avg_time_mst} shows the time duration for computing the distributed MCCST, which shows that computing distributed MCCST could be done in real time with large number of robots.

The average distance to target region and perturbation after convergence is calculated from 10 runs for each batch of robots with up to 100 robots. In Figure~\ref{fig:avg_distance}, the average distance to target with MCCST is significantly smaller (closer to target region) than with static connectivity graph. The distance also decreases as the number of robots increases, since only a limited number of robots are needed to maintain connectivity, which enables more robots to rendezvous to the target locations. However, for the other two methods with imposed static connectivity topology, the average distance increases with the number of robots. Figure~\ref{fig:perturbation} shows the result of average perturbation. Our method gives much smaller perturbation on average. Note that the result from preserving initial MST gives much worse result than the other two, because the initial MST edges could give huge deviation from the optimal control outputs as behaviors progressed, while the full connectivity graph gives larger number of constrain edges to keep, so that some are canceled out with each other. Nevertheless, our distributed MCCST method always computes the minimum connectivity constraints, thus outperforming the other two methods significantly.  

\section{Conclusion} 
\label{sec:conclusion}
In this paper, we developed the bilevel optimization based minimum connectivity maintenance framework for behavior mixing. We proposed a distributed Minimum Connectivity Constraint Spanning Tree (MCCST) algorithm to compute provably minimum global and subgroup connectivity constraints in real-time.
By formulating the invoked connectivity constraints and safety constraints using safety and connectivity barrier certificates, the robots controllers are minimally modified from the original controllers with dynamic and possibly discontinuous communication topology. 
Experimental results show that our method is scalable and computation efficient to large number of robots. 
Future work includes 
the incremental computation of MCCST to more efficiently handle the robots joining or leaving the team dynamically.

%% Use plainnat to work nicely with natbib. 

\bibliographystyle{IEEEtran}
\bibliography{ref}

\end{document}